%% file: main.tex
\documentclass{article}

\usepackage[final,nonatbib]{neurips_2019}

\usepackage[utf8]{inputenc} %
\usepackage[T1]{fontenc}    %
\usepackage{hyperref}       %
\usepackage{url}            %
\usepackage{booktabs}       %
\usepackage{amsfonts}       %
\usepackage{nicefrac}       %
\usepackage{microtype}      %

\usepackage{algorithmicx,algpseudocode,algorithm}

\usepackage{microtype}
\usepackage{graphicx}
\usepackage{booktabs} %
\usepackage{enumitem} 
\usepackage{bm}
\usepackage{nicefrac}
\usepackage{placeins}
\usepackage{pifont}%
\newcommand{\cmark}{\ding{51}}%
\newcommand{\xmark}{\ding{55}}%
\usepackage{wrapfig,tikz}
\usepackage{amsmath,longtable,fancyhdr,booktabs,multirow,graphicx,float}
\usepackage{adjustbox, bigstrut, tabularx, multirow, makecell, diagbox}
\usepackage{amssymb,xcolor,amsthm}
\usepackage{color}
\usepackage{colortbl}
\usepackage{theoremref}
\usepackage{subcaption}

\newcommand{\mbf}[1]{\mathbf{#1}}

\newcommand{\mbb}[1]{\mathbb{#1}}

\newcommand{\ud}{\mathrm{d}}

\newcommand{\mcal}{\mathcal}
\newcommand{\norm}[1]{\left\lVert#1\right\rVert}

\newtheorem{lemma}{Lemma}
\newtheorem{theorem}{Theorem}

\newtheorem{example}{Example}

\newtheorem{proposition}{Proposition}

\newtheorem{innercustomthm}{Theorem}
\newenvironment{customthm}[1]
{\renewcommand\theinnercustomthm{#1}\innercustomthm}
{\endinnercustomthm}

\newtheorem{innercustomprop}{Proposition}
\newenvironment{customprop}[1]
{\renewcommand\theinnercustomprop{#1}\innercustomprop}
{\endinnercustomprop}

\newcommand{\be}{\begin{equation}}
	\newcommand{\ee}{\end{equation}}

\definecolor{Gray}{gray}{0.85}
\definecolor{LightCyan}{rgb}{0.88,1,1}

\newcolumntype{a}{>{\columncolor{Gray}}c}
\newcolumntype{b}{>{\columncolor{white}}c}

\makeatletter
\usepackage{xspace}
\def\@onedot{\ifx\@let@token.\else.\null\fi\xspace}
\DeclareRobustCommand\onedot{\futurelet\@let@token\@onedot}

\newcommand{\eqnref}[1]{Eq\onedot~\eqref{#1}}
\newcommand{\figref}[1]{Fig\onedot~\ref{#1}}
\newcommand{\algoref}[1]{Algorithm~\ref{#1}}

\newcommand{\secref}[1]{Section~\ref{#1}}
\newcommand{\tabref}[1]{Tab\onedot~\ref{#1}}
\newcommand{\thmref}[1]{Theorem~\ref{#1}}

\newcommand{\lemref}[1]{Lemma~\ref{#1}}
\newcommand{\propref}[1]{Proposition~\ref{#1}}

\newcommand{\bfx}{\mathbf{x}}
\newcommand{\bfA}{\mathbf{A}}
\newcommand{\bfB}{\mathbf{B}}
\newcommand{\bfW}{\mathbf{W}}
\newcommand{\bfV}{\mathbf{V}}
\newcommand{\bfM}{\mathbf{M}}
\newcommand{\bfb}{\mathbf{b}}

\newcommand{\bfz}{\mathbf{z}}

\newcommand{\bft}{\mathbf{t}}
\newcommand{\bfzero}{\mathbf{0}}
\newcommand{\frakm}{\mathfrak{L}}

\newcommand{\bftheta}{{\boldsymbol{\theta}}}

\newcommand{\bfy}{\mathbf{y}}

\def\eg{\emph{e.g}\onedot}

\def\ie{\emph{i.e}\onedot}

\def\cf{\emph{cf}\onedot}

\def\wrt{w.r.t\onedot}

\usepackage[textsize=tiny]{todonotes}

\title{MintNet: Building Invertible Neural Networks with Masked Convolutions}

\author{%
  Yang Song\thanks{Equal contribution.} \\
  Stanford University \\
  \texttt{yangsong@cs.stanford.edu} \\
  \And
  Chenlin Meng\footnotemark[1] \\
  Stanford University \\
  \texttt{chenlin@cs.stanford.edu} \\
  \And
  Stefano Ermon \\
  Stanford University \\
  \texttt{ermon@cs.stanford.edu} \\
}

\begin{document}
\maketitle
\input{abstract}
\input{intro}

\input{pre}
\input{method}

\input{exp}

\input{conclude}

\bibliography{flow}
\newpage
\appendix
\input{app}
\end{document}

%% file: abstract.tex
\begin{abstract}
We propose a new way of constructing invertible neural networks by combining simple building blocks with a novel set of composition rules. This leads to a rich set of invertible architectures, including those similar to ResNets. Inversion is achieved with a locally convergent iterative procedure that is parallelizable and very fast in practice. 
Additionally, the determinant of the Jacobian can be computed analytically and efficiently, enabling their generative use as flow models. 
To demonstrate their flexibility, we show that our invertible neural networks are competitive with ResNets on MNIST and CIFAR-10 classification. 
When trained as generative models, our invertible networks achieve competitive likelihoods on MNIST, CIFAR-10 and ImageNet 32$\times$32, with bits per dimension of 0.98, 3.32 and 4.06 respectively.
\end{abstract}

%% file: intro.tex
\section{Introduction}
Invertible neural networks
have many applications in machine learning. 
They have been employed to investigate representations of deep classifiers~\cite{jacobsen2018irevnet}, understand the cause of adversarial examples~\cite{jacobsen2018excessive}, learn transition operators for MCMC~\cite{song2017nice,levy2018generalizing}, create generative models that are directly trainable by maximum likelihood~\cite{nvp,dinh2016density,maf,glow,FFJORD,i-resnet}, and perform approximate inference~\cite{rezende15variational,kingma2016iaf}.

Many applications of invertible neural networks require that both inverting the network and computing the Jacobian determinant be efficient.
While typical neural networks are not invertible, achieving these properties often imposes restrictive constraints
to the architecture. 
For example, planar flows~\cite{rezende15variational} and Sylvester flow~\cite{berg2018sylvester} constrain the number of hidden units to be smaller than the input dimension. NICE~\cite{dinh2016density} and Real NVP~\cite{nvp} rely on dimension partitioning heuristics and specific architectures such as coupling layers, which could make training more difficult~\cite{i-resnet}. Methods like FFJORD~\cite{FFJORD}, i-ResNets~\cite{i-resnet} have fewer architectural constraints. However, their Jacobian determinants have to be approximated, which is problematic if repeatedly performed at training time as in flow models.

In this paper, we propose a new method of constructing invertible neural networks which are flexible, efficient to invert, and whose Jacobian can be computed \emph{exactly} and efficiently.
We use triangular matrices as our basic module. %
Then, we provide a set of composition rules to recursively build more complex non-linear modules from the basic module, and show that the composed modules are invertible as long as their Jacobians are non-singular. As in previous work~\cite{nvp,maf},
the Jacobians of our modules are triangular, allowing efficient determinant computation. The inverse of these modules can be obtained by an efficiently parallelizable fixed-point iteration method, making the cost of inversion comparable to that of an i-ResNet~\cite{i-resnet} block. %

Using our composition rules and masked convolutions as the basic triangular building block, we construct a rich set of invertible modules to form a deep invertible neural network. The architecture of our proposed invertible network closely follows that of ResNet~\cite{he2016deep}---the state-of-the-art architecture of discriminative learning. We call our model \textbf{M}asked \textbf{In}ver\textbf{t}ible \textbf{Net}work (MintNet). To demonstrate the capacity of MintNets, we first test them on image classification. We found that a MintNet classifier achieves 99.6\% accuracy on MNIST, matching the performance of a ResNet with a similar architecture. On CIFAR-10, it achieves 91.2\% accuracy, comparable to the 92.6\% accuracy of ResNet. When using MintNets as generative models, they achieve the new state-of-the-art results of bits per dimension (bpd) on uniformly dequantized images. Specifically, MintNet achieves bpd values of 0.98, 3.32, and 4.06 on MNIST, CIFAR-10 and ImageNet 32$\times$32, while former best published results are 0.99 (FFJORD~\cite{FFJORD}), 3.35 (Glow~\cite{glow}) and 4.09 (Glow) respectively.
Moreover, MintNet uses fewer parameters and less computational resources. Our MNIST model uses 30\% fewer parameters than FFJORD~\cite{FFJORD}. For CIFAR-10 and ImageNet 32$\times$32, MintNet uses 60\% and 74\% fewer parameters than the corresponding Glow~\cite{glow} models.
When training on dataset such as CIFAR-10, MintNet required 2 GPUs for approximately 5 days, while FFJORD~\cite{FFJORD} used 6 GPUs for approximately 5 days, and Glow~\cite{glow} used 8 GPUs for approximately 7 days. %

%% file: pre.tex
\section{Background}
Consider a neural network $f: \mbb{R}^D \rightarrow \mbb{R}^L$ that maps a data point $\bfx \in \mbb{R}^D$ to a latent representation $\bfz \in \mbb{R}^L$. When for every $\bfz \in \mbb{R}^L$ there exists a unique $\bfx \in \mbb{R}^D$ such that $f(\bfx) = \bfz$, we call $f$ an invertible neural network. There are several basic properties of invertible networks. First, when $f(\bfx)$ is continuous, a necessary condition for $f$ to be invertible is $D = L$. Second, if $f_1: \mbb{R}^D \rightarrow \mbb{R}^D$ and $f_2: \mbb{R}^D \rightarrow \mbb{R}^D$ are both invertible, $f = f_2 \circ f_1$ will also be invertible. In this work, we mainly consider applications of invertible neural networks to classification and generative modeling.

\subsection{Classification with invertible neural networks}

Neural networks for classification are usually not invertible
because the number of classes $L$ is usually different from the input dimension $D$. Therefore, when discussing invertible neural networks for classification, we separate the classifier into two parts $f = f_2 \circ f_1$: feature extraction $\bfz = f_1(\bfx)$ and classification $\bfy = f_2(\bfz)$, where $f_2$ is usually the softmax function. We say the classifier is invertible when $f_1$ is invertible.
Invertible classifiers are arguably more interpretable, because a prediction can be traced down by inverting latent representations~\cite{jacobsen2018irevnet, jacobsen2018excessive}.

\subsection{Generative modeling with invertible neural networks}
An invertible network $f: \bfx \in \mbb{R}^D \mapsto \bfz \in \mbb{R}^D$ can be used to warp a complex probability density $p(\bfx)$ to a simple base distribution $\pi(\bfz)$ (\eg, a multivariate standard Gaussian)~\cite{dinh2016density,nvp}.
Under the condition that both $f$ and $f^{-1}$ are differentiable, the densities of $p(\bfx)$ and $\pi(\bfz)$ are related by the following change of variable formula
\begin{align}
    \log p(\bfx) = \log \pi(\bfz) + \log |\operatorname{det}( J_f(\bfx) )|, \label{eqn:flow}
\end{align}
where $J_f(\bfx)$ denotes the Jacobian of $f(\bfx)$ and we require $J_f(\bfx)$ to be non-singular so that $\log |\operatorname{det}( J_f(\bfx) )|$ is well-defined. Using this formula, $p(\bfx)$ can be easily computed if the Jacobian determinant $\operatorname{det}(J_f(\bfx))$ is cheaply computable and $\pi(\bfz)$ is known.

Therefore, an invertible neural network $f_\bftheta(\bfx)$ implicitly defines a normalized density model $p_\bftheta(\bfx)$, which can be directly trained by maximum likelihood. The invertibility of $f_\bftheta$ is critical to fast sample generation. Specifically, in order to generate a sample $\bfx$ from $p_\bftheta(\bfx)$, we can first draw $\bfz \sim \pi(\bfz)$, and warp it back through the inverse of $f_\bftheta$ to obtain $\bfx = f_\bftheta^{-1}(\bfz)$.

Note that multiple invertible models $f_1, f_2, \cdots, f_K$ can be stacked together to form a deeper invertible model $f = f_K \circ \cdots \circ f_2 \circ f_1$, without much impact on the inverse and determinant computation. This is because we can sequentially invert each component, \ie, $f^{-1} = f_1^{-1} \circ f_2^{-1} \circ \cdots \circ f_K^{-1}$, and the total Jacobian determinant equals the product of each individual Jacobian determinant, \ie, $|\operatorname{det}(J_f)| = |\operatorname{det}(J_{f_1})||\operatorname{det}(J_{f_2})|\cdots |\operatorname{det}(J_{f_K})|$.

%% file: method.tex
\section{Building invertible modules compositionally}\label{sec:composition}
In this section, we discuss how simple blocks like masked convolutions can be composed to build invertible modules that allow efficient, parallelizable inversion and determinant computation. To this end, we first introduce the basic building block of our models. Then, we propose a set of composition rules to recursively build up complex non-linear modules with triangular Jacobians. Next, we prove that these composed modules are invertible as long as their Jacobians are non-singular. Finally, we discuss how these modules can be inverted efficiently using numerical methods.

\subsection{The basic module}\label{sec:basic}
We start from considering linear transformations $f(\bfx) = \bfW \bfx + \bfb$, with $\bfW \in \mbb{R}^{D\times D}$, and $\bfb \in \mbb{R}^D$. %
For a general $\bfW$, computing its Jacobian determinant requires $O(D^3)$ operations. We therefore choose $\bfW$ to be a triangular matrix. In this case, the Jacobian determinant $\operatorname{det}(J_f(\bfx)) = \operatorname{det}(\bfW)$ is the product of all diagonal entries of $\bfW$, and the computational complexity is reduced to $O(D)$. The linear function $f(\bfx) = \bfW\bfx + \bfb$ with $\bfW$ being triangular is our \emph{basic module}. %

\begin{figure}
    \centering
    \includegraphics[width=0.9 \textwidth]{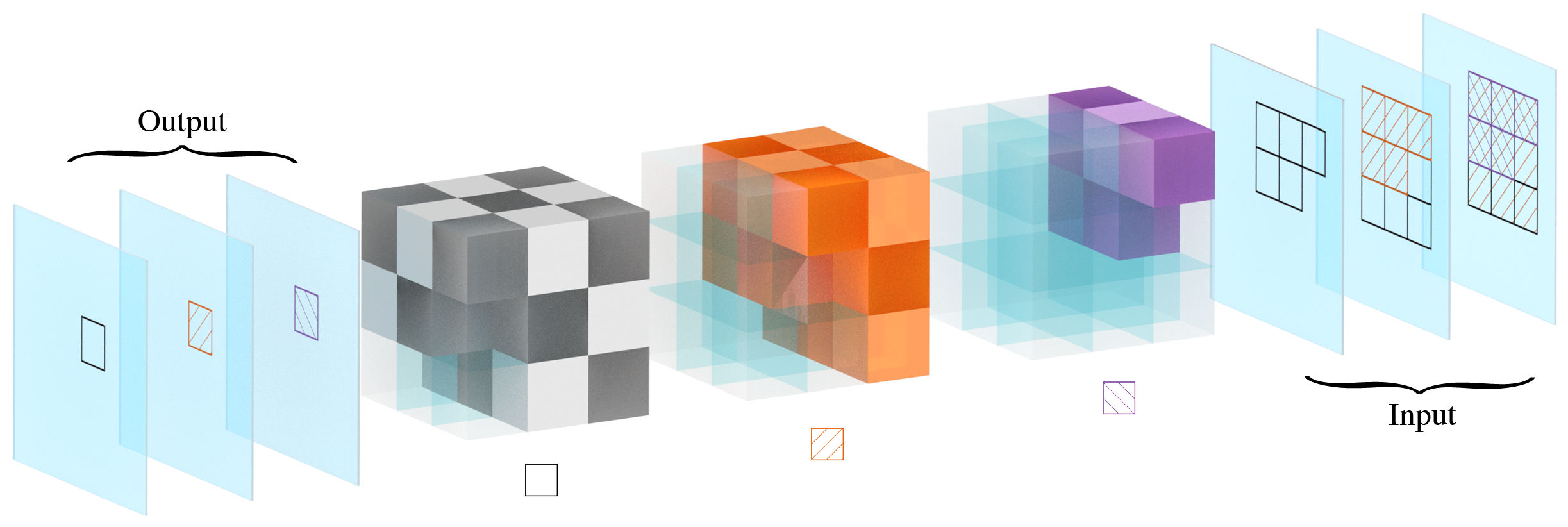}
    \caption{Illustration of a masked convolution with 3 filters and kernel size $3 \times 3$. Solid checkerboard cubes inside each filter represent unmasked weights, while the transparent blue blocks represent the weights that have been masked out. The receptive field of each filter on the input feature maps is indicated by regions shaded with the pattern (the colored square) below the corresponding filter.}
    \label{fig:masked_conv}
\end{figure}

\paragraph{Masked convolutions.} Convolution is a special type of linear transformation that is very effective for image data. The triangular structure of the basic module can be achieved using \emph{masked} convolutions (\eg, causal convolutions in PixelCNN~\cite{oord2016pixel}). We provide the formula of our masks in Appendix~\ref{app:masked_conv} and  an illustration of a $3\times3$ masked convolution with $3$ filters in \figref{fig:masked_conv}. Intuitively, the causal structure of the filters (ordering of the pixels) enforces a triangular structure.

\subsection{The calculus of building invertible modules}\label{sec:calculus}
Complex non-linear invertible functions can be constructed from our basic modules in two steps. First, we follow several composition rules so that the composed module has a triangular Jacobian. Next, we impose appropriate constraints so that the module is invertible. To simplify the discussion, we only consider modules with lower triangular Jacobians here, and we note that it is straightforward to extend the analysis to modules with upper triangular Jacobians. 

The following proposition summarizes several rules to compositionally build new modules with triangular Jacobians using existing ones.
\begin{proposition}\label{prop:calculus}
Define $\mcal{F}$ as the set of all continuously differentiable functions whose Jacobian is lower triangular. Then $\mcal{F}$ contains the basic module in \secref{sec:basic}, and is closed under the following composition rules.
\begin{itemize}
    \item \textbf{Rule of addition}. $f_1 \in \mcal{F} \wedge f_2\in\mcal{F} \Rightarrow \lambda f_1 + \mu f_2 \in \mcal{F}$, where $\lambda, \mu \in \mbb{R}$. %
    \item \textbf{Rule of composition}. $f_1 \in \mcal{F} \wedge f_2 \in \mcal{F} \Rightarrow f_2 \circ f_1 \in \mcal{F}$. A special case is $f \in \mcal{F} \Rightarrow h \circ f \in \mcal{F}$, where $h(\cdot)$ is a continuously differentiable non-linear activation function that is applied element-wise.%
\end{itemize}
\end{proposition}
The proof of this proposition is straightforward and deferred to Appendix \ref{app:proof}. By repetitively applying the rules in \propref{prop:calculus}, our basic linear module can be composed to construct 
complex non-linear modules having continuous and triangular Jacobians. Note that besides our linear basic modules, other functions with triangular and continuous Jacobians can also be made more expressive using the composition rules. For example, the layers of dimension partitioning models (\eg, NICE~\cite{dinh2016density}, Real NVP~\cite{nvp}, Glow~\cite{glow}) and autoregressive flows (\eg, MAF~\cite{maf}) all have continuous and triangular Jacobians and therefore belong to $\mcal{F}$. Note that the rule of addition in \propref{prop:calculus} preserves triangular Jacobians but not invertibility. Therefore, we need additional constraints if we want the composed functions to be invertible.

\begin{figure}
    \centering
    \includegraphics[width=0.9 \textwidth]{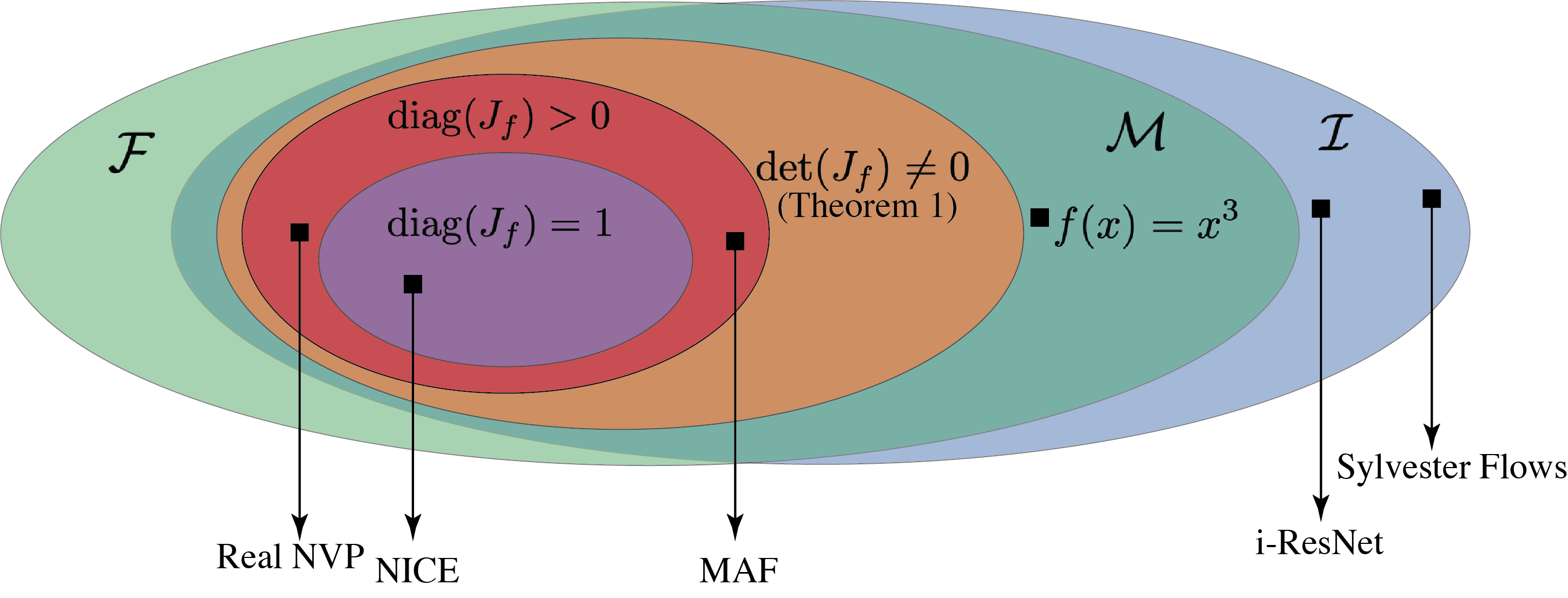}
    \caption{Venn Diagram relationships between invertible functions ($\mcal{I}$), the function sets of $\mcal{F}$ and $\mcal{M}$, functions that meet the conditions of \thmref{thm:invert} ($\operatorname{det}(J_f) \neq 0$), functions whose Jacobian is triangular and Jacobian diagonals are strictly positive ($\operatorname{diag}(J_f) > 0$), functions whose Jacobian is triangular and Jacobian diagonals are all 1s ($\operatorname{diag}(J_f) = 1$).}
    \label{fig:venn}
\end{figure}

Next, we state the condition for $f \in \mcal{F}$ to be invertible, and denote the invertible subset of $\mcal{F}$ as $\mcal{M}$. %
\begin{theorem}\label{thm:invert}
If $f\in\mcal{F}$ and $J_f(\bfx)$ is non-singular for all $\bfx$ in the domain, then $f$ is invertible.
\end{theorem}

\begin{proof}
A proof can be found in Appendix~\ref{app:proof}.
\end{proof}

The non-singularity of $J_f(\bfx)$ constraint in \thmref{thm:invert} is natural in the context of generative modeling. This is because in order for \eqnref{eqn:flow} to make sense, $\log |\operatorname{det}(J_f)|$ has to be well-defined, which requires $J_f(\bfx)$ to be non-singular.

In many cases, \thmref{thm:invert} can be easily used to check and enforce the invertibility of $f \in \mcal{F}$. For example, the layers of autoregressive flow models and dimension partitioning models can all be viewed as elements of $\mcal{F}$ because they are continuously differentiable and have triangular Jacobians. Since the diagonal entries of their Jacobians are always strictly positive and hence non-singular, we can immediately conclude that they are invertible with \thmref{thm:invert}, thus generalizing their model-specific proofs of invertibility. 

In \figref{fig:venn}, we provide a Venn Diagram to illustrate the set of functions that satisfy the condition of \thmref{thm:invert}. As depicted by the orange set labeled by $\operatorname{det}(J_f) \neq 0$, \thmref{thm:invert} captures a subset of $\mcal{M}$ where the Jacobians of functions are non-singular so that the change of variable formula is usable. Note the condition in Theorem 1 is sufficient but not necessary.
For example, $f(x) = x^3 \in \mcal{M}$ is invertible, but $J_f(x = 0) = 3x^2 |_{x=0} = 0$ is singular. Many previous invertible models with special architectures, such as NICE, Real NVP, and MAF, can be viewed as elements belonging to subsets of $\operatorname{det}(J_f) \neq 0$.

\subsection{Efficient inversion of the invertible modules}
In this section, we show that when the conditions in \thmref{thm:invert} hold, not only do we know that $f$ is invertible ($f \in \mcal{M}$), but also we have a fixed-point iteration method to invert $f$ with strong theoretical guarantees and good performance in practice.

\begin{algorithm}
	\caption{Fixed-point iteration method for computing $f^{-1}(\bfz)$.}
	\label{alg:solver}
	\begin{algorithmic}[1]
	    \Require{$T, \alpha$} \Comment{$T$ is the number of iterations; $0<\alpha <2$ is the step size.}
	    \State{Initialize $\bfx_0$}
        \For{$t \gets 1$ to $T$}
            \State{Compute $f(\bfx_{t-1})$}
            \State{Compute $\operatorname{diag}(J_f(\bfx_{t-1}))$}
            \State{$\bfx_t \gets \bfx_{t-1} - \alpha \operatorname{diag}(J_f(\bfx_{t-1}))^{-1}(f(\bfx_{t-1}) - \bfz)$}
        \EndFor
        \item[]
        \Return{$\bfx_T$}
	\end{algorithmic}
\end{algorithm}

The pseudo-code of our proposed inversion algorithm is described in \algoref{alg:solver}. Theoretically, we can prove that this method is locally convergent---as long as the initial value is close to the true value, the method is guaranteed to find the correct inverse. We formally summarize this result in \thmref{thm:converge}.
\begin{theorem}\label{thm:converge}
The iterative method of \algoref{alg:solver} is locally convergent whenever $0 < \alpha < 2$.
\end{theorem}
\begin{proof}%
We provide a more rigorous proof in Appendix~\ref{app:proof}.
\end{proof}

In practice, the method is also easily parallelizable on GPUs, making the cost of inverting $f\in\mcal{M}$ similar to that of an i-ResNet~\cite{i-resnet} layer. Within each iteration, the computation is mostly matrix operations that can be vectorized and run efficiently in parallel. Therefore, the time cost will be roughly proportional to the number of iterations, \ie, $O(T)$. As will be shown in our experiments, \algoref{alg:solver} converges fast and usually the error quickly becomes negligible when $T \ll D$. This is in stark contrast to existing methods of inverting autoregressive flow models such as MAF~\cite{maf}, where $D$ univariate equations need to be solved sequentially, requiring at least $O(D)$ iterations. There are also other approaches for inverting $f$. For example, the bisection method is guaranteed to converge globally, but its computational cost is $O(D)$, and is usually much more expensive than \algoref{alg:solver}. Note that as discussed earlier, autoregressive flow models can also be viewed as special cases of our framework. Therefore, \algoref{alg:solver} is also applicable to inverting autoregressive flow models and could potentially result in large improvements of sampling speed.

\section{Masked Invertible Networks}\label{sec:invnet}
We show that techniques developed in \secref{sec:composition} can be used to build our Masked Invertible Network (MintNet). First, we discuss how we compose several masked convolutions to form the Masked Invertible Layer (Mint layer). Next, we stack multiple Mint layers to form a deep neural network, \ie, the MintNet. Finally, we compare MintNets with several existing invertible architectures.

\subsection{Building the Masked Invertible Layer}\label{sec:layer}
We construct an invertible module in $\mcal{M}$ that serves as the basic layer of our MintNet. This invertible module, named Mint layer, is defined as
\begin{align}
    \frakm(\bfx) = \bft \odot \bfx + \sum_{i=1}^K \bfW^3_i h\bigg(\sum_{j=1}^K \bfW^2_{ij} h(\bfW^1_j \bfx + \mbf{b}^1_j) + \mbf{b}^2_{ij}\bigg) + \mbf{b}^3_i, \label{eqn:block}
\end{align}
where $\odot$ denotes the elementwise multiplication, $\{\bfW^1_i\}|_{i=1}^{K}$, $\{\bfW^2_{ij}\}|_{1 \leq i,j\leq K}$, and $\{\bfW^3_i\}|_{i=1}^K$ are all lower triangular matrices with additional constraints to be specified later, and $\bft > \bfzero$. Additionally, Mint layers use a monotonic activation function $h$, so that $h' \geq 0$. Common choices of $h$ include ELU~\cite{clevert2015fast}, tanh and sigmoid. Note that every individual weight matrix has the same size, and the 3 groups of weights $\{\bfW^1_i\}|_{i=1}^{K}$, $\{\bfW^2_{ij}\}|_{1 \leq i,j\leq K}$ and $\{\bfW^3_i\}|_{i=1}^K$ can be implemented with 3 masked convolutions (see Appendix~\ref{app:masked_conv}). We design the form of $\frakm(\bfx)$ so that it resembles a ResNet / i-ResNet block that also has 3 convolutions with $K \times C$ filters, with $C$ being the number of channels of $\bfx$. When using \algoref{alg:solver} to invert Mint layers, we initialize $\bfx_0 = \bfz \odot \frac{1}{\bft}$.

From \propref{prop:calculus} in \secref{sec:calculus}, we can easily conclude that $\frakm \in \mcal{F}$. Now, we consider additional constraints on the weights so that $\frakm \in \mcal{M}$, \ie, it is invertible. Note that the analytic form of its Jacobian is
\begin{align}
    J_\frakm(\bfx) = \sum_{i=1}^K \bfW_i^3 \bfA_i \sum_{j=1}^K \bfW_{ij}^2 \bfB_j \bfW_j^1 + \bft, \label{eqn:jacobian}
\end{align}
with $\bfA_i = \operatorname{diag}(h'\big(\sum_{j=1}^K \bfW^2_{ij} h(\bfW^1_j \bfx + \mbf{b}^1_j) + \mbf{b}^2_{ij}\big)) \geq \bfzero$, $\bfB_j = \operatorname{diag}(h'(\bfW^1_j \bfx + \mbf{b}^1_j)) \geq \bfzero$, and $\bft > \bfzero$. Therefore, once we impose the following constraint
\begin{align}
    \operatorname{diag}(\bfW_i^3)\operatorname{diag}(\bfW_{ij}^2)\operatorname{diag}(\bfW_j^1) \geq \bfzero,  \forall 1 \leq i,j\leq K, \label{eqn:constraint}
\end{align}
we have $\operatorname{diag}(J_{\frakm}(\bfx)) > \bfzero$, which satisfies the condition of \thmref{thm:invert} and as a consequence we know $\frakm \in \mcal{M}$. %
In practice, the constraint \eqnref{eqn:constraint} can be easily implemented. For all $1 \leq i, j \leq K$, we impose no constraint on $\bfW_i^3$ and $\bfW_j^1$, but replace $\bfW_{ij}^2$ with $\bfV_{ij}^2 =\bfW_{ij}^2\operatorname{sign}(\operatorname{diag}(\bfW_{ij}^2))\operatorname{sign}(\operatorname{diag}(\bfW_i^3 \bfW_j^1))$. 
Note that $\operatorname{diag}(\bfV_{ij}^2)$ has the same signs as $\operatorname{diag}(\bfW_i^3)\operatorname{diag}(\bfW_j^1)$ and therefore $\operatorname{diag}(\bfW_i^3)\operatorname{diag}(\bfV_{ij}^2)\operatorname{diag}(\bfW_j^1) \geq \bfzero$. Moreover, $\bfV_{ij}^2$ is almost everywhere differentiable \wrt $\bfW_{ij}^2$, which allows gradients to backprop through.

\subsection{Constructing the Masked Invertible Network}
In this section, we introduce design choices that help stack multiple Mint layers together to form an expressive invertible neural network, namely the MintNet. The full MintNet is constructed by stacking the following paired Mint layers and squeezing layers.
\paragraph{Paired Mint layers.} As discussed above, our Mint layer $\frakm(\bfx)$ always has a triangular Jacobian. To maximize the expressive power of our invertible neural network, it is undesirable to constrain the Jacobian of the network to be triangular since this limits capacity and will cause blind spots in the receptive field of masked convolutions. We thus always pair two Mint layers together---one with a lower triangular Jacobian and the other with an upper triangular Jacobian, so that the Jacobian of the paired layers is not triangular, and blind spots can be eliminated.

\paragraph{Squeezing layers.} Subsampling is important for enlarging the receptive field of convolutions. However, common subsampling operations such as pooling and strided convolutions are usually not invertible. Following \cite{nvp} and \cite{i-resnet}, we use a ``squeezing'' operation to reshape the feature maps  so that they have smaller resolution but more channels. After a squeezing operation, the height and width will decrease by a factor of $k$
, but the number of channels will increase by a factor of $k^2$.  This procedure is invertible and the Jacobian is an identity matrix. Throughout the paper, we use $k=2$.

\subsection{Comparison to other approaches}\label{sec:connection}
In what follows we compare MintNets to several existing methods for developing invertible architectures. We will focus on architectures with a tractable Jacobian determinant. However, we note that there are models (\cf, \cite{gomez2017reversible,mackay2018reversible,revnet}) that allow fast inverse computation but do not have tractable Jacobian determinants. Following~\cite{i-resnet}, we also provide some comparison in \tabref{tab:compare} (see Appendix~\ref{app:tables}).

\subsubsection{Models based on identities of determinants} Some identities can be used to speed up the computation of determinants if the Jacobians have special structures. For example, in Sylvester flow~\cite{berg2018sylvester}, the invertible transformation has the form $f(\bfx) \triangleq \bfx + \bfA h(\bfB\bfx + \bfb)$, where $h(\cdot)$ is a nonlinear activation function, $\bfA \in \mbb{R}^{D\times M}$, $\bfB \in \mbb{R}^{M \times D}$, $\bfb \in \mbb{R}^M$ and $M \leq D$. By Sylvester's determinant identity, $\operatorname{det}(J_f(\bfx))$ can be computed in $O(M^3)$, which is much less than $O(D^3)$ if $M \ll D$. However, the requirement that $M$ is small becomes a bottleneck of the architecture and limits its expressive power. Similarly, Planar flow~\cite{rezende15variational} uses the matrix determinant lemma, but has an even narrower bottleneck. 

The form of $\frakm(\bfx)$ bears some resemblance to Sylvester flow. However, we improve the capacity of Sylvester flow in two ways. First, we add one extra non-linear convolutional layer. Second, we avoid the bottleneck that limits the maximum dimension of latent representations in Sylvester flow.

\subsubsection{Models based on dimension partitioning} NICE~\cite{dinh2016density}, Real NVP~\cite{nvp}, and Glow~\cite{glow} all depend on an affine coupling layer. Given $d < D$, $\bfx$ is first partitioned into two parts $\bfx = [\bfx_{1:d}; \bfx_{d+1:D}]$. The coupling layer is an invertible transformation, defined as $f:\bfx \mapsto \bfz, \quad \bfz_{1:d} = \bfx_{1:d}, \quad \bfz_{d+1:D} = \bfx_{d+1:D} \odot \exp(s(\bfx_{1:d})) + t(\bfx_{1:d})$, where $s(\cdot)$ and $t(\cdot)$ are two arbitrary functions. %
However, the partitioning of $\bfx$ relies on heuristics, and the performance is sensitive to this choice (\cf, \cite{glow,i-resnet}). In addition, the Jacobian of $f$ is a triangular matrix with diagonal $[\mbf{1}_d; \exp(s(\bfx_{1:d}))]$. In contrast, the Jacobian of MintNets has more flexible diagonals---without being partially restricted to $1$'s.

\subsubsection{Models based on autoregressive transformations} By leveraging autoregressive transformations, the Jacobian can be made triangular. For example, MAF~\cite{maf} defines the invertible tranformation as $f: \bfx \mapsto \bfz, \quad z_i = \mu(\bfx_{1:i-1}) + \sigma(\bfx_{1:i-1}) x_i$, where $\mu(\cdot) \in \mbb{R}$ and $\sigma(\cdot) \in \mbb{R}^+$. Note that $f^{-1}(\bfz)$ can be obtained by sequentially solving $x_i$ based on previous solutions $\bfx_{1:i-1}$. Therefore, a na\"{i}ve approach requires $\Omega(D)$ computations for inverting autoregressive models. Moreover, the architecture of $f$ is only an affine combination of autoregressive functions with $\bfx$. In contrast, MintNets are inverted with faster fixed-point iteration methods, and the architecture of MintNets is arguably more flexible.

\subsubsection{Free-form invertible models} Some work proposes invertible transformations whose Jacobians are not limited by special structures. For example, FFJORD~\cite{FFJORD} uses a continuous version of change of variables formula~\cite{chen2018neural} where the determinant is replaced by trace. Unlike MintNets, FFJORD needs an ODE solver to compute its value and inverse, and uses a stochastic estimator to approximate the trace. Another work is i-ResNet~\cite{i-resnet} which constrains the Lipschitz-ness of ResNet layers to make it invertible. Both i-ResNet and MintNet use ResNet blocks with 3 convolutions. The inverse of i-ResNet can be obtained efficiently by a parallelizable fixed-point iteration method, which has comparable computational cost as our \algoref{alg:solver}. However, unlike MintNets whose Jacobian determinants are exact, the log-determinant of Jacobian of an i-ResNet must be approximated by truncating a power series and estimating each term with stochastic estimators. 

\subsubsection{Other models using masked convolutions}
Emerging convolutions~\cite{hoogeboom2019emerging} and MaCow~\cite{ma2019macow} improve the Glow architecture by replacing $1 \times 1$ convolutions in the original Glow model with masked convolutions similar to those employed in MintNets. Emerging convolutions and MaCow are both inverted using forward/back substitutions designed for inverting triangular matrices, which requires the same number of iterations as the input dimension. In stark contrast, MintNets use a fixed-point iteration method (\algoref{alg:solver}) for inversion, which is similar to i-ResNet and requires substantially fewer iterations than the input dimension. For example, our method of inversion takes 120 iterations to converge on CIFAR-10, while inverting emerging convolutions will need 3072 iterations. In other words, our inversion can be 25 times faster on powerful GPUs. Additionally, the architecture of MintNet is very different. The architectures of \cite{hoogeboom2019emerging} and \cite{ma2019macow} are both built upon Glow. In contrast, MintNet is a ResNet architecture where normal convolutions are replaced by causal convolutions.

%% file: exp.tex
\section{Experiments}\label{sec:exp}
In this section, we evaluate our MintNet architectures on both image classification and density estimation. We focus on three common image datasets, namely MNIST, CIFAR-10 and ImageNet 32$\times$32. We also empirically verify that \algoref{alg:solver} can provide accurate solutions within a small number of iterations. We provide more details about settings and model architectures in Appendix~\ref{app:network}.

\subsection{Classification}
To check the capacity of MintNet and understand the trade-off of invertibility, we test its classification performance on MNIST and CIFAR-10, and compare it to a ResNet with a similar architecture. %

On MNIST, MintNet achieves a test accuracy of 99.6\%, which is the same as that of the ResNet. On CIFAR-10, MintNet reaches 91.2\% test accuracy while ResNet reaches 92.6\%. 
Both MintNet and ResNet achieve 100\% training accuracy on MNIST and CIFAR-10 datasets. This indicates that MintNet has enough capacity to fit all data labels on the training dataset, and the invertible representations learned by MintNet are comparable to representations learned by non-invertible networks in terms of generalizability. Note that the small degradation in classification accuracy is also observed in other invertible networks. For example, depending on the Lipschitz constant, the gap between test accuracies of i-ResNet and ResNet can be as large as 1.92\% on CIFAR-10.  %

\subsection{Density estimation and verification of invertibility}\label{sec:exp:invert}

In this section, we demonstrate the superior performance of MintNet on density estimation by training it as a flow generative model. In addition, we empirically verify that \algoref{alg:solver} can accurately produce the inverse using a small number of iterations. We show that samples can be efficiently generated from MintNet by inverting each Mint layer with \algoref{alg:solver}.
\paragraph{Density estimation.}
In \tabref{tab:bpd_table}, we report bits per dimension (bpd) on MNIST, CIFAR-10, and ImageNet 32$\times$32 datasets. It is notable that MintNet sets the new records of bpd on all three datasets. Moreover, when compared to previous best models, our MNIST model uses 30\% fewer parameters than FFJORD, and our CIFAR-10 and ImageNet 32$\times$32 models respectively use 60\% and 74\% fewer parameters than Glow. When trained on datasets such as CIFAR-10, MintNet requires 2 GPUs for approximately five days, while FFJORD is trained on 6 GPUs for five days, and Glow on 8 GPUs for seven days. 
Note that all values in \tabref{tab:bpd_table} are with respect to the continuous distribution of uniformly dequantized images, and results of models that view images as discrete distributions are not directly comparable (\eg, PixelCNN~\cite{oord2016pixel}, IAF-VAE~\cite{kingma2016iaf}, and Flow++~\cite{ho2019flow}). To show that MintNet learns semantically meaningful representations of images, we also perform latent space interpolation similar to the interpolation experiments in Real NVP (see Appendix~\ref{app:interpolation}).

\begin{table}
 \caption{MNIST, CIFAR-10, ImageNet 32$\times$32 bits per dimension (bpd) results. Smaller values are better. $^\dagger$Result not directly comparable because ZCA preprocssing was used.} \label{tab:bpd_table}
\begin{center}
    \begin{tabular}{p{5cm} c c c}
        \toprule
        Method & MNIST & CIFAR-10 & ImageNet 32$\times$32\\
        \midrule
        NICE~\cite{dinh2016density} &4.36 &4.48$^\dagger$ & -\\
        MAF~\cite{maf} &1.89 &4.31 & -\\
        Real NVP~\cite{nvp} &1.06 &3.49 & 4.28\\
        Glow~\cite{glow} &1.05 &3.35 & 4.09\\
        FFJORD~\cite{FFJORD} &0.99 &3.40 & -\\
        i-ResNet~\cite{i-resnet} &1.06 &3.45 & - \\
        \midrule
        MintNet (ours) &\textbf{0.98} &\textbf{3.32} & \textbf{4.06}\\
        \bottomrule
    \end{tabular} 
\end{center}
\end{table}

\paragraph{Verification of invertibility.} We first examine the performance of \algoref{alg:solver} by measuring the reconstruction error of MintNets. We compute the inverse of MintNet by sequentially inverting each Mint layer with \algoref{alg:solver}. We used grid search to select the step size $\alpha$ in \algoref{alg:solver} and chose $\alpha=3.5, 1.1, 1.15$ respectively for MNIST, CIFAR-10 and ImageNet 32$\times$32. An interesting fact is for MNIST, $\alpha=3.5$ actually works better than other values of $\alpha$ within $(0,2)$, even though it does not have the theoretical gurantee of local convergence. As \figref{fig:newton} shows, the normalized $L_2$ reconstruction error converges within $120$ iterations for all datasets considered. Additionally, \figref{fig:recon} demonstrates that the reconstructed images look visually indistinguishable to true images. 

\paragraph{Samples.} Using \algoref{alg:solver}, we can generate samples efficiently by computing the inverse of MintNets. We use the same step sizes as in the reconstruction error analysis, and run \algoref{alg:solver} for 120 iterations for all three datasets. We provide uncurated samples in \figref{fig:samples}, and more samples can be found in Appendix~\ref{app:samples}. In addition, we compare our sampling time to that of the other models (see \tabref{tab:sample_time} in Appendix~\ref{app:tables}). Our sampling method has comparable speed as i-ResNet. It is approximately 5 times faster than autoregressive sampling on MNIST, and is roughly 25 times faster on CIFAR-10 and ImageNet 32$\times$32.
\begin{figure}%
    \centering
    \begin{subfigure}[b]{0.3\textwidth}
        \includegraphics[width=\textwidth]{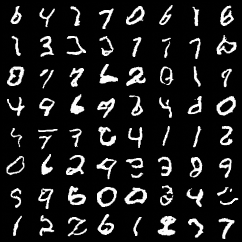}
        \caption{MNIST}
        \label{fig:mnist}
    \end{subfigure}
    \begin{subfigure}[b]{0.3\textwidth}
        \includegraphics[width=\textwidth]{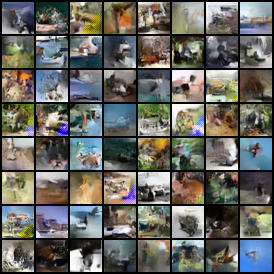}
        \caption{CIFAR-10}
        \label{fig:cifar10}
    \end{subfigure}
    \begin{subfigure}[b]{0.3\textwidth}
        \includegraphics[width=\textwidth]{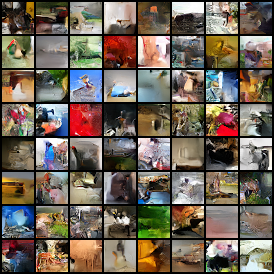}
        \caption{ImageNet-32$\times$32}
        \label{fig:imagenet32}
    \end{subfigure}
    \caption{Uncurated samples on MNIST, CIFAR-10, and ImageNet 32$\times$32 datasets.}
    \label{fig:samples}
\end{figure}
\begin{figure}%
    \centering
    \begin{subfigure}[b]{0.5\textwidth}
        \includegraphics[width=\textwidth]{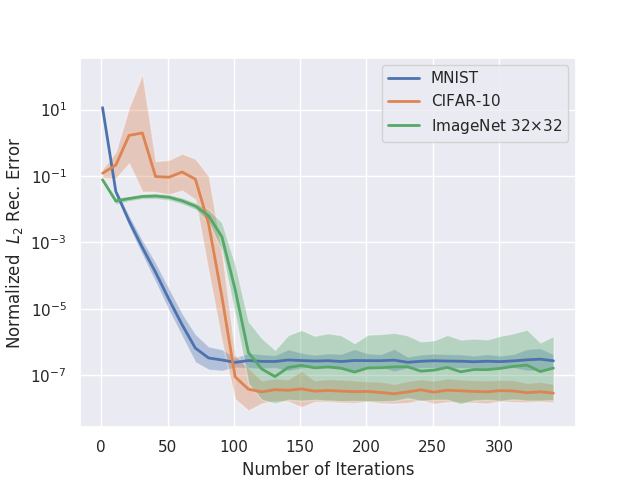}
        \caption{Reconstruction error analysis.}
        \label{fig:newton}
    \end{subfigure}
    \begin{subfigure}[b]{0.45\textwidth}
        \includegraphics[width=\textwidth]{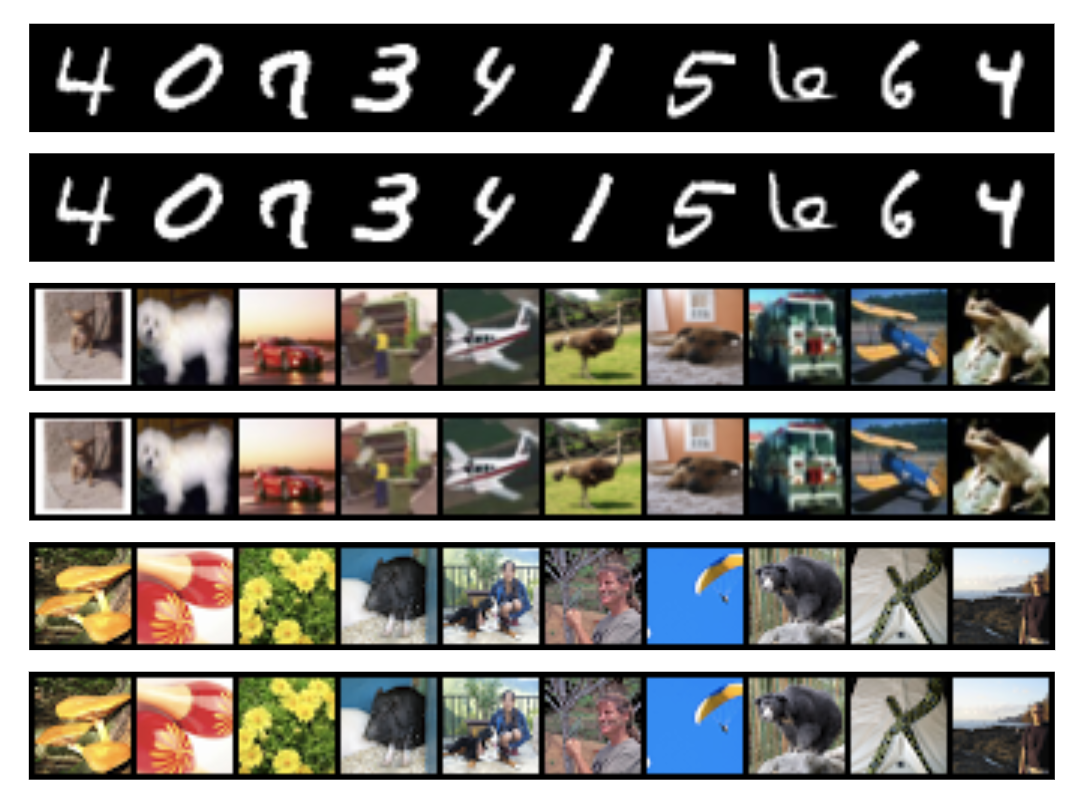}
        \caption{Reconstructed images.}
        \label{fig:recon}
    \end{subfigure}
    \caption{Accuracy analysis of \algoref{alg:solver} on MNIST, CIFAR-10, and ImageNet 32$\times$32 datasets. Each curve in (a) represents the mean value of normalized reconstruction errors for 128 images. The 2nd, 4th and 6th rows in (b) are reconstructions, while other rows are original images.}
    \label{fig:invert}
\end{figure}

%% file: conclude.tex
\section{Conclusion}
We propose a new method to compositionally construct invertible modules that are flexible, efficient to invert, and with a tractable Jacobian. Starting from linear transformations with triangular matrices, we apply a set of composition rules to recursively build new modules that are non-linear and more expressive~(\propref{prop:calculus}). We then show that the composed modules are invertible as long as their Jacobians are non-singular (\thmref{thm:invert}), and propose an efficiently parallelizable numerical method (\algoref{alg:solver}) with theoretical guarantees (\thmref{thm:converge}) to compute the inverse. The Jacobians of our modules are all triangular, which allows efficient and exact determinant computation.

As an application of this idea, we use masked convolutions as our basic module. Using our composition rules, we compose multiple masked convolutions together to form a module named Mint layer, following the architecture of a ResNet block. To enforce its invertibility, we constrain the masked convolutions to satisfy the condition of \thmref{thm:invert}. We show that multiple Mint layers can be stacked together to form a deep invertible network which we call MintNet. The architecture can be efficiently inverted using a fixed point iteration algorithm (\algoref{alg:solver}). Experimentally, we show that MintNet performs well on MNIST and CIFAR-10 classification. Moreover, when trained as a generative model, MintNet achieves new state-of-the-art performance on MNIST, CIFAR-10 and ImageNet 32$\times$32.

\subsection*{Acknowledgements}
This research was supported by Intel Corporation, Amazon AWS, TRI, NSF (\#1651565, \#1522054, \#1733686), ONR  (N00014-19-1-2145), AFOSR (FA9550-
19-1-0024).

%% file: app.tex
\section{Proofs}\label{app:proof}
\paragraph{Notations.} Let $J_f(\bfx)$ denote the Jacobian of $f$ evaluated at $\bfx$. We use $[f(\bfx)]_i$ to denote the $i$-th component of the vector-valued function $f$, and $[J_f(\bfx)]_{ij}$ to denote the $ij$-th entry of $J_f(\bfx)$. We further use $\bfx_i$ to denote the $i$-th component of the input vector $\bfx \in \mbb{R}^D$, and $\frac{\partial [f(\bfx)]_i}{\partial \bfx_j}\big|_{\bfx = \bft}$ to denote the partial derivative of $[f(\bfx)]_i$ \wrt $\bfx_j$, evaluated at $\bfx = \bft$.

\begin{customprop}{\ref{prop:calculus}}
Define $\mcal{F}$ as the set of all continuously differentiable functions whose Jacobian is lower triangular. Then $\mcal{F}$ contains the basic module in \secref{sec:basic}, and is closed under the following composition rules.
\begin{itemize}
   \item \textbf{Rule of addition}. $f_1 \in \mcal{F} \wedge f_2\in\mcal{F} \Rightarrow \lambda f_1 + \mu f_2 \in \mcal{F}$, where $\lambda, \mu \in \mbb{R}$. %
    \item \textbf{Rule of composition}. $f_1 \in \mcal{F} \wedge f_2 \in \mcal{F} \Rightarrow f_2 \circ f_1 \in \mcal{F}$. A special case is $f \in \mcal{F} \Rightarrow h \circ f \in \mcal{F}$, where $h(\cdot)$ is a continuously differentiable non-linear activation function that is applied element-wisely.
\end{itemize}
\end{customprop}
\begin{proof}
    Since the basic modules have the form $f(\bfx) = \mbf{W}\bfx + \bfb$, where $\mbf{W}$ is a lower triangular matrix, we immediately know that $f$ is continuously differentiable and $J_f$ is lower triangular, therefore $f \in \mcal{F}$. Next, we prove the closeness properties of $\mcal{F}$ one by one.
    \begin{itemize}
        \item \textbf{Rule of addition}. $f = \lambda f_1 + \mu f_2$ is continuously differentiable, and $J_f$ is lower triangular. This is because $\nicefrac{\partial f}{\partial \bfx} = \nicefrac{\partial (\lambda f_1 + \mu f_2)}{\partial \bfx} = \lambda\nicefrac{\partial f_1}{\partial \bfx} + \mu \nicefrac{\partial f_2}{\partial \bfx}$, and both $\nicefrac{\partial f_1}{\partial \bfx}$ and $\nicefrac{\partial f_2}{\partial \bfx}$ are continuous and lower triangular.
        
        \item \textbf{Rule of composition}. $f = f_2 \circ f_1$ is continuously differentiable and has a lower triangular Jacobian. This is because $\nicefrac{\partial f}{\partial \bfx} =\nicefrac{\partial (f_2 \circ f_1)}{\partial \bfx} = \nicefrac{\partial f_2}{\partial \bfx}\big|_{\bfx = f_1(\bfx)} \nicefrac{\partial f_1}{\partial \bfx}$, and both $\nicefrac{\partial f_2}{\partial \bfx}$ and $\nicefrac{\partial f_1}{\partial \bfx}$ are continuous and lower triangular. As a special case, we choose $f_1 = h$, where $h$ is a continuously differentiable univariate function. Since the Jacobian of $h$ is diagonal and continuous, we have $h \in \mcal{F}$. Therefore $h \circ f_2 \in \mcal{F}$ holds true for all $f_2 \in \mcal{F}$.
    \end{itemize}
\end{proof}

The following two lemmas will be very helpful for proving \thmref{thm:invert}.
\begin{lemma}
\label{condition1}
$J_f(\bfx)$ is lower triangular for all $\bfx \in \mbb{R}^D$ implies $[f(\bfx)]_i$ is a function of $\bfx_1,...,\bfx_i$, and does not depend on $\bfx_{i+1}, \cdots, \bfx_D$.
\end{lemma}
\begin{proof}
Due to the fact that $J_f(\bfx)$ is lower triangular, we have $[J_{f}(\bfx)]_{i,j}=\frac{\partial [f(\bfx)]_{i}}{\partial \bfx_{j}}=0$ for any $j>i$. When $\bfx_1,...,\bfx_{j-1},\bfx_{j+1},...,\bfx_D$ are fixed, we have
\begin{align}
 [f(\bfx_1,...,\bfx_{j-1},\bfx_j,\bfx_{j+1},\bfx_D)]_i &=[f(\bfx_1,...,\bfx_{j-1},0,\bfx_{j+1},...,\bfx_D)]_i +\int_{0}^{\bfx_j}\frac{\partial [f(\bft)]_{i}}{\partial \bft_{j}} \ud \bft_j\\
&=[f(\bfx_1,...,\bfx_{j-1},0,\bfx_{j+1},...,\bfx_D)]_i.
\end{align}

This implies that $[f(\bfx)]_i$ does not depend on $\bfx_{j}$ for any $j>i$. In other words, $f(\bfx)$ is only a function of $\bfx_1,...,\bfx_i$.
\end{proof}
\begin{lemma}
\label{condition2}
$\operatorname{diag}(J_f(\bfx) J_f(\bfzero)) > \bfzero$ implies that for any $1\leq i\leq n$, either (i) $\forall \bfx \in \mbb{R}^D: [J_f(\bfx)]_{ii}>0$ or (ii) $\forall \bfx \in \mbb{R}^D: [J_f(\bfx)]_{ii}<0$. That is, $[f(\bfx)]_i$ is monotonic \wrt $\bfx_i$ when $\bfx_1,\cdots,\bfx_{i-1}$ are fixed.
\end{lemma}
\begin{proof}
Clearly $\operatorname{diag}(J_f(\bfx) J_f(\bfzero)) > \bfzero$ is equivalent to $\forall 1\leq i \leq n, \bfx \in \mbb{R}^D: [J_f(\bfx)]_{ii}[J_f(\bfzero)]_{ii} = \frac{\partial [f(\bfx)]_i}{\partial \bfx_i}\frac{\partial [f(\bfx)]_i}{\partial \bfx_i}\big|_{\bfx=\bfzero}>0$. This means for any $\bfx \in \mbb{R}^D$, $[J_f(\bfx)]_{ii} = \frac{\partial [f(\bfx)]_i}{\partial \bfx_i}\neq 0$ and it shares the same sign with $[J_f(\bfzero)]_{ii}=\frac{\partial [f(\bfx)]_i}{\partial \bfx_i}\big|_{\bfx=\bfzero}$, a constant that is either strictly positive or strictly negative. This further implies that when $\bfx_1, \cdots, \bfx_{i-1}$ are fixed, $\frac{\partial [f(\bfx)]_i}{\partial \bfx_i}$ is either strictly positive or strictly negative for all $\bfx_{i}\in \mathbb{R}$, and $[f(\bfx)]_i$ is therefore monotonic \wrt $\bfx_i$.
\end{proof}
\begin{customthm}{\ref{thm:invert}}
If $f\in\mcal{F}$ and $J_f(\bfx)$ is non-singular for all $\bfx$ in the domain, then $f$ is invertible.
\end{customthm}

\begin{proof}
Assume without loss of generality that $J_f(\bfx)$ is lower triangular. We first prove that $\operatorname{diag}(J_f(\bfx) J_f(\bfzero)) > \bfzero$ by contradiction. Assuming $\operatorname{diag}(J_f(\bfx) J_f(\bfzero)) \leq \bfzero$, then $\exists 1\leq i \leq n, \bfx' \in \mbb{R}^D$ such that $[J_f(\bfx')]_{ii} [J_f(\bfzero)]_{ii} \leq 0$. Because $J_f(\bfx)$ is always triangular and non-singular, we immediately conclude that $[J_f(\bfx')]_{ii} [J_f(\bfzero)]_{ii} < 0$. Assume without loss of generality that $[J_f(\bfzero)]_{ii} > 0$ and $[J_f(\bfx')]_{ii} < 0$. Then, by the intermediate value theorem, we know that $\exists t \in (0,1)$ such that $[J_f(t \bfx')]_{ii} = 0$, which contradicts that fact that $J_f(\bfx)$ is always non-singular.

Next, we prove that for all $\bfz$ in the range of $f(\bfx)$, there exists a unique $\bfx$ such that $f(\bfx) = \bfz$. To obtain $\bfx_1$, we only need to solve $[f(\bfx)]_1=\bfz_1$, which is an equation of variable $\bfx_1$, as concluded from \lemref{condition1}. Since \lemref{condition2} implies that $[f(\bfx)]_{1}$ is monotonic \wrt $\bfx_1$, we know that $[f(\bfx)]_{1}$ has a unique inverse $\bfx_1$ whenever $\bfz_1$ is in the range of $[f(\bfx)]_1$. Now assume we have already obtained $\bfx_1,...,\bfx_k$, where $k \geq 1$. In this case, \lemref{condition1} asserts that $[f(\bfx)]_{k+1} = \bfz_{k+1}$ is an equation of variable $\bfx_{k+1}$. Again \lemref{condition2} implies that $[f(\bfx)]_{k+1}$ is a monotonic function of $\bfx_{k+1}$ given $\bfx_1, \cdots, \bfx_k$, which implies further that $[f(\bfx)]_{k+1} = \bfz_{k+1}$ has a unique solution $\bfx_{k+1}$ whenever $\bfz_{k+1}$ is in the range of $[f(\bfx)]_{k+1}$. By induction, we can solve for $\bfx_1, \bfx_2, \cdots, \bfx_D$ by repetitively employing this procedure, which concludes that $f^{-1}(\bfz)=(\bfx_1,...,\bfx_{D})^\intercal$ exists, and can be determined uniquely. %

\end{proof}

\begin{customthm}{\ref{thm:converge}}
The iterative method of \algoref{alg:solver} is locally convergent whenever $0 < \alpha < 2$.
\end{customthm}
\begin{proof}
    Let $\bfz$ be any value in the range of $f(\bfx)$ and $g(\bfx; \alpha, \bfz) \triangleq \bfx - \alpha \operatorname{diag}(J_f(\bfx))^{-1}[f(\bfx) - \bfz]$, where $\operatorname{diag}(A)^{-1}$ denotes a diagonal matrix whose diagonal entries are the reciprocals of those of $A$. The iterative method of \algoref{alg:solver} can be written as $\bfx_t = g(\bfx_{t-1}; \alpha, \bfz)$. Because of \thmref{thm:invert}, there exists a unique $\bfx^* \in \mbb{R}^D$ such that $f(\bfx^*) = \bfz$, in which case $g(\bfx^*; \alpha, \bfz) = \bfx^*$. Applying the product rule, we have
    \begin{align*}
        J_g(\bfx^*; \alpha, \bfz) = I - \alpha \operatorname{diag}(J_f(\bfx^*))^{-1} J_f(\bfx^*),
    \end{align*}
    where $J_g(\bfx^*; \alpha, \bfz)$ denotes the Jacobian of $g(\bfx; \alpha, \bfz)$ evaluated at $\bfx^*$. Since $J_f(\bfx^*)$ is triangular, $J_g(\bfx^*;\alpha,\bfz)$ will also be triangular. Therefore, the only eigenvalue of $J_g(\bfx^*; \alpha, \bfz)$ is $1 - \alpha$, due to the fact that the only solution to the equation system $\operatorname{det}(\lambda I - J_g(\bfx^*; \alpha, \bfz)) = (\lambda- 1 + \alpha)^D = 0$ is $\lambda = 1 - \alpha$. Since $0 < \alpha < 2$, the spectral radius of $J_g(\bfx^*; \alpha, \bfz)$ satisfies $\rho(J_g(\bfx^*; \alpha, \bfz)) = |1 - \alpha| < 1$. Then the Ostrowski Theorem (\cf, Theorem 10.1.3. in \cite{ortega1970iterative}) shows that the sequence $\{ \bfx_1, \bfx_2, \cdots, \bfx_t \}$ obtained by $\bfx_t = g(\bfx_{t-1}; \alpha, \bfz)$ converges locally to $\bfx^*$ as $t \rightarrow \infty$.
\end{proof}

\section{Masked convolutions}\label{app:masked_conv}
Convolution is a special type of linear transformation that proves to be very effective for image data. The basic invertible module can be implemented using masked convolutions (\eg, causal convolutions in PixelCNN~\cite{oord2016pixel}). Consider a 2D convolutional layer with $C_\text{in}$ input feature maps, $C_\text{out}$ filters, a kernel size of $R \times R$ and a zero-padding of $\lfloor \nicefrac{R}{2} \rfloor$. We assume $R$ is an odd integer and $C_\text{out} = C_\text{in}$ so that the input and output of the convolutional layer have the same shape. Let $\bfW \in \mbb{R}^{C_\text{out} \times C_\text{in} \times R \times R}$ be the weight tensor of this layer. We define a mask $\mbf{M}\in\{0,1\}^{C_\text{out} \times C_\text{in} \times R \times R}$ that satisfies
\begin{align}
    \mbf{M}[i, j, m, n] = \begin{cases}
    0, &\quad \text{if $i < j$ or $i = j \wedge m > \lfloor \nicefrac{R}{2} \rfloor $ or $i = j \wedge m = \lfloor \nicefrac{R}{2} \rfloor \wedge n > \lfloor \nicefrac{R}{2} \rfloor$,}\\
    1, &\quad \text{Otherwise}.
    \end{cases}\label{eqn:mask}
\end{align}
The masked convolution then uses $\bfM \odot \bfW$ as the weight tensor. In \figref{fig:masked_conv}, we provide an illustration on a $3\times3$ masked convolution with $3$ filters.

In MintNet, $\frakm(\bfx)$ is efficiently implemented with 3 masked convolutional layers. The weights and masks are denoted as $(\bfW^1, \bfM^1)$, $(\bfW^2, \bfM^2)$ and $(\bfW^3, \bfM^3)$, which separately correspond to $\{\bfW^1_i\}_{i=1}^K, \{\bfW^2_{ij}\}_{1\leq i,j\leq K}, \{\bfW^3_j\}_{j=1}^K$ in \eqnref{eqn:block}. Let $C$ be the number of input feature maps, and suppose the kernel size is $R \times R$. The shapes of $\bfW^1$, $\bfW^2$ and $\bfW^3$ are respectively $(K C, C, R, R)$, $(K C, K C, R, R)$ and $(C, K C, R, R)$. The masks of them are simple concatenations of copies of the mask in \eqnref{eqn:mask}. For instance, $\bfM^1$ consists of $K$ copies of \eqnref{eqn:mask}, and $\bfM^2$ consists of $K \times K$ copies. Using masked convolutions, $\frakm(\bfx)$ can be concisely written as
\begin{align}
    \resizebox{0.93\textwidth}{!}{$\frakm(\bfx) = \bft \odot \bfx + (\bfW^3 \odot \bfM^3) \circledast h\bigg( (\bfW^2 \odot \bfM^2) \circledast h\big( (\bfW^1 \odot \bfM^1) \circledast \bfx + \mbf{b}^1\big) + \mbf{b}^2 \bigg) + \mbf{b}^3$}, \label{eqn:conv_block}
\end{align}
where $\bfb^1, \bfb^2, \bfb^3$ are biases, and $\circledast$ denotes the operation of discrete 2D convolution.

\section{Interpolation of hidden representations}
\label{app:interpolation}
\begin{figure}[!t]
\centering
\includegraphics[width=0.3\textwidth]{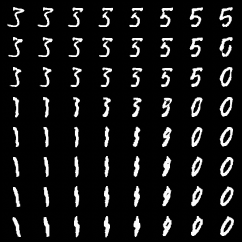}
\includegraphics[width=0.3\textwidth]{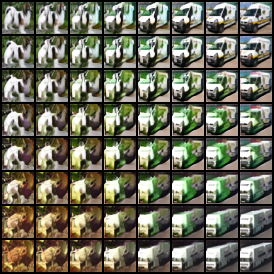}
\includegraphics[width=0.3\textwidth]{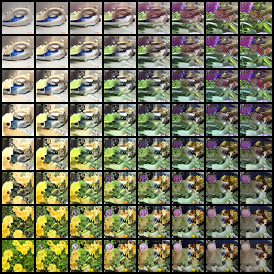}
\caption{MintNet interpolation of hidden representation. \textbf{Left:} MintNet MNIST latent space interpolation. \textbf{Middle:} MintNet CIFAR-10 latent space interpolation. \textbf{Right:} MintNet ImageNet 32$\times$32 latent space interpolation.}
\label{fig:interpolation}
\end{figure}
Given four images $\bfx_1,\bfx_2,\bfx_3,\bfx_4$ in the dataset, let $\bfz_i=f(\bfx_i)$, where $i=1,2,3,4$, be the corresponding features in the feature domain. Similar to \cite{nvp}, in the feature domain, we define
\begin{equation}
    \bfz=\cos(\phi)(\cos(\phi')\bfz_1+\sin(\phi')\bfz_2)+\sin(\phi)(\cos(\phi')\bfz_3+\sin(\phi')\bfz_4)
\end{equation}
where $\bfx$-axis corresponds to $\phi'$, $\bfy$-axis corresponds to $\phi$, and both $\phi$ and $\phi'$ range over $\{0,\frac{\pi}{14},...,\frac{7\pi}{14}\}$. We then transform $\bfz$ back to the image domain by taking $f^{-1}(\bfz)$. Interpolation results are shown in \figref{fig:interpolation}.

\section{Experiment setup and network architecture}
\label{app:network}
\paragraph{Hyperparameter tuning and computation infrastructure.} We use the standard train/test split of MNIST, CelebA and CIFAR-10. We tune our models by observing its training bpd. For density estimation on CIFAR-10 and ImageNet 32$\times$32, the models were run on two Titan XP GPUs. In other cases the model was run on one Titan XP GPU.

\paragraph{Classification setup.} Following \cite{i-resnet}, we pad the images to 16 channels with zeros. This corresponds to the first convolution in ResNet which increases the number of channels to 16. Both ResNet and our MintNet are trained with AMSGrad~\cite{reddi2018on} for 200 epochs with the cosine learning rate schedule~\cite{loshchilov2016sgdr} and an initial learning rate of 0.001. Both networks use a batch size of 128.

\paragraph{Classification architecture.} The ResNet contains 38 pre-activation residual blocks~\cite{he2016identity}, and each block has three $3\times 3$ convolutions. The architecture is divided into 3 stages, with 16, 64 and 256 filters respectively. Our MintNet uses 19 grouped invertible layers, which include a total of 38 residual invertible layers, each having three $3\times 3$ convolutions. Batch normalization is applied before each invertible layer. Note that batch normalization does not affect the invertibility of our network, because during test time it uses fixed running average and standard deviation and is an invertible operation. We use 2 squeezing blocks at the same position where ResNet applies subsampling, and matches the number of filters used in ResNet. To produce the logits for classification, both MintNet and ResNet first apply global average pooling and then use a fully connected layer (see \tabref{tab:classification}).

\paragraph{Density estimation setup.}
We mostly follow the settings in \cite{maf}. All training images are dequantized and transformed using the logit transformation. Networks are trained using AMSGrad~\cite{amsgrad}. On MNIST, we decay the learning rate by a factor of 10 at the 250th and 350th epoch, and train for 400 epochs. On CIFAR-10, we train with cosine learning rate decay for a total of 200 epochs. On ImageNet 32$\times$32, we train with cosine learning rate decay for a total of 350k steps. All initial learning rates are 0.001. 

\paragraph{Density estimation architecture.}
For density estimation on MNIST, we use 20 paired Mint layers with 45 filters each. For both CIFAR-10 and ImageNet 32$\times$32, we use 21 paired Mint layers, each of which has 255 filters. For all the three datasets, two squeezing operations are used and are distributed evenly across the network  (see \tabref{tab:mnist_density} and \tabref{tab:cifar_imagenet_density}). 

\paragraph{Tuning the step size for sampling.}
We perform grid search to find hyperparamter $\alpha$ for \algoref{alg:solver} using a minibatch of 128 images. More specifically, we start from $\alpha=1$ to 5 with a step size 0.5 for MNIST, CIFAR-10, and ImageNet 32$\times$32, and compute the normalized $L_2$ reconstruction error with respect to the number of iterations. The normalized $L_2$ error is defined as $\norm{\bfx - \bfy}_2^2 / D$, where $\bfx \in \mbb{R}^D$ and $\bfy\in\mbb{R}^D$ are two image vectors corresponding to the original and reconstructed images. We find that the algorithm converges most quickly when $\alpha$ is in intervals $[3,4]$, $[1,2]$ and $[1,2]$ for MNIST, CIFAR-10 and ImageNet 32$\times$32 respectively. Then we perform a second round grid search on the corresponding interval with a step size 0.05. In this case, we are able to find the best $\alpha$, that is $\alpha=3.5,1.1,1.15$ for the corresponding datasets.

\paragraph{Verification of invertibility.}
To verify the invertibility of MintNet, we study the normalized $L_2$ reconstruction error for MNIST, CIFAR-10 and ImageNet 32$\times$32. The $L_2$ reconstruction error is computed for 128 images on all three datasets. We plot the exponential of the mean log reconstruction errors in \figref{fig:invert}. The shaded area corresponds to the exponential of the standard deviation of log reconstruction errors.

\begin{table}
	\small
	\centering
	\caption{MintNet image classification network architecture.}
	\begin{tabular}{l|c|c}
		\hline\bigstrut
		\bf Name & \bf Configuration & \bf Replicate Block \\
		\hline\bigstrut
		\multirow{12}{*}{\shortstack{Paired Mint Block1\\ with Batch Normalization}}
		& batch normalization\\
		& $3\times 3$ lower triangular masked convolution, 1 filter
		& \multirow{10}{*}{$\times 6$}\\
		& leaky relu activation& \\
		& $3\times 3$ lower triangular masked convolution, $1$ filter\\
		& leaky relu activation& \\
		& $3\times 3$ lower triangular masked convolution, $1$ filter \\
		& batch normalization\\
		& $3\times 3$ upper triangular masked convolution,$1$ filter \\
		& leaky relu activation& \\
		& $3\times 3$ upper triangular masked convolution, $1$ filter \\
		& leaky relu activation& \\
		& $3\times 3$ upper triangular masked convolution, $1$ filter \\
		\hline\bigstrut
		Squeezing Layer & 
		$2 \times 2$ squeezing layer & {---} \\
		\hline\bigstrut
		\multirow{12}{*}{\shortstack{Paired Mint Block2\\ with Batch Normalization}} 
		& batch normalization\\
		& $3\times 3$ lower triangular masked convolution, 1 filter
		& \multirow{10}{*}{$\times 6$}\\
		& leaky relu activation& \\
		& $3\times 3$ lower triangular masked convolution, $1$ filter\\
		& leaky relu activation& \\
		& $3\times 3$ lower triangular masked convolution, $1$ filter \\
		& batch normalization\\
		& $3\times 3$ upper triangular masked convolution,$1$ filter \\
		& leaky relu activation& \\
		& $3\times 3$ upper triangular masked convolution, $1$ filter \\
		& leaky relu activation& \\
		& $3\times 3$ upper triangular masked convolution, $1$ filter \\
		\hline\bigstrut
		Squeezing Layer & 
		$2 \times 2$ squeezing layer & {---} \\
		\hline\bigstrut
		\multirow{12}{*}{\shortstack{Paired Mint Block3\\ with Batch Normalization}} 
		& batch normalization\\
		& $3\times 3$ lower triangular masked convolution, 1 filter
		& \multirow{10}{*}{$\times 7$}\\
		& leaky relu activation& \\
		& $3\times 3$ lower triangular masked convolution, $1$ filter\\
		& leaky relu activation& \\
		& $3\times 3$ lower triangular masked convolution, $1$ filter \\
		& batch normalization\\
		& $3\times 3$ upper triangular masked convolution,$1$ filter \\
		& leaky relu activation& \\
		& $3\times 3$ upper triangular masked convolution, $1$ filter \\
		& leaky relu activation& \\
		& $3\times 3$ upper triangular masked convolution, $1$ filter \\
		\hline\bigstrut
		\multirow{3}{*}{Output Layer}
		& average pooling
		& \multirow{3}{*}{---}\\
		& fully connected layer\\
		& softmax layer\\
		\hline
	\end{tabular}
\label{tab:classification}
\end{table}

\begin{table}
	\small
	\centering
	\caption{MintNet MNIST density estimation network architecture.}
	\begin{tabular}{l|c|c}
		\hline\bigstrut
		\bf Name & \bf Configuration & \bf Replicate Block \\
		\hline\bigstrut
		\multirow{10}{*}{Paired Mint Block1} 
		& $3\times 3$ lower triangular masked convolution, 45 filters
		& \multirow{10}{*}{$\times 6$}\\
		& elu activation& \\
		& $3\times 3$ lower triangular masked convolution, $45$ filters\\
		& elu activation& \\
		& $3\times 3$ lower triangular masked convolution, $45$ filters \\
		& $3\times 3$ upper triangular masked convolution,$45$ filters \\
		& elu activation& \\
		& $3\times 3$ upper triangular masked convolution, $45$ filters \\
		& elu activation& \\
		& $3\times 3$ upper triangular masked convolution, $45$ filters \\
		\hline\bigstrut
		Squeezing Layer & 
		$2 \times 2$ squeezing layer & {---} \\
		\hline\bigstrut
		\multirow{10}{*}{Paired Mint Block2} 
			& $3\times 3$ lower triangular masked convolution, 45 filters
		& \multirow{10}{*}{$\times 6$}\\
		& elu activation& \\
		& $3\times 3$ lower triangular masked convolution, $45$ filters\\
		& elu activation& \\
		& $3\times 3$ lower triangular masked convolution, $45$ filters \\
		& $3\times 3$ upper triangular masked convolution,$45$ filters \\
		& elu activation& \\
		& $3\times 3$ upper triangular masked convolution, $45$ filters \\
		& elu activation& \\
		& $3\times 3$ upper triangular masked convolution, $45$ filters \\
		\hline\bigstrut
		Squeezing Layer & 
		$2 \times 2$ squeezing layer & {---} \\
		\hline\bigstrut
		\multirow{10}{*}{Paired Mint Block3} 
		& $3\times 3$ lower triangular masked convolution, 45 filters
		& \multirow{10}{*}{$\times 8$}\\
		& elu activation& \\
		& $3\times 3$ lower triangular masked convolution, $45$ filters\\
		& elu activation& \\
		& $3\times 3$ lower triangular masked convolution, $45$ filters \\
		& $3\times 3$ upper triangular masked convolution,$45$ filters \\
		& elu activation& \\
		& $3\times 3$ upper triangular masked convolution, $45$ filters \\
		& elu activation& \\
		& $3\times 3$ upper triangular masked convolution, $45$ filters \\
		\hline
	\end{tabular}
\label{tab:mnist_density}
\end{table}

\begin{table}
	\small
	\centering
	\caption{MintNet CIFAR-10 and Imagenet 32$\times$32 density estimation network architecture.}
	\begin{tabular}{l|c|c}
		\hline\bigstrut
		\bf Name & \bf Configuration & \bf Replicate Block \\
		\hline\bigstrut
		\multirow{10}{*}{Paired Mint Block1} 
		& $3\times 3$ lower triangular masked convolution, 85 filters
		& \multirow{10}{*}{$\times 7$}\\
		& elu activation& \\
		& $3\times 3$ lower triangular masked convolution, $85$ filters\\
		& elu activation& \\
		& $3\times 3$ lower triangular masked convolution, $85$ filters \\
		& $3\times 3$ upper triangular masked convolution,$85$ filters \\
		& elu activation& \\
		& $3\times 3$ upper triangular masked convolution, $85$ filters \\
		& elu activation& \\
		& $3\times 3$ upper triangular masked convolution, $85$ filters \\
		\hline\bigstrut
		Squeezing Layer & 
		$2 \times 2$ squeezing layer & {---} \\
		\hline\bigstrut
		\multirow{10}{*}{Paired Mint Block2} 
			& $3\times 3$ lower triangular masked convolution, 85 filters
		& \multirow{10}{*}{$\times 7$}\\
		& elu activation& \\
		& $3\times 3$ lower triangular masked convolution, $85$ filters\\
		& elu activation& \\
		& $3\times 3$ lower triangular masked convolution, $85$ filters \\
		& $3\times 3$ upper triangular masked convolution,$85$ filters \\
		& elu activation& \\
		& $3\times 3$ upper triangular masked convolution, $85$ filters \\
		& elu activation& \\
		& $3\times 3$ upper triangular masked convolution, $85$ filters \\
		\hline\bigstrut
		Squeezing Layer & 
		$2 \times 2$ squeezing layer & {---} \\
		\hline\bigstrut
		\multirow{10}{*}{Paired Mint Block3} 
		& $3\times 3$ lower triangular masked convolution, 85 filters
		& \multirow{10}{*}{$\times 7$}\\
		& elu activation& \\
		& $3\times 3$ lower triangular masked convolution, $85$ filters\\
		& elu activation& \\
		& $3\times 3$ lower triangular masked convolution, $85$ filters \\
		& $3\times 3$ upper triangular masked convolution,$85$ filters \\
		& elu activation& \\
		& $3\times 3$ upper triangular masked convolution, $85$ filters \\
		& elu activation& \\
		& $3\times 3$ upper triangular masked convolution, $85$ filters \\
		\hline
	\end{tabular}
\label{tab:cifar_imagenet_density}
\end{table}

\clearpage
\section{Additional tables}
\label{app:tables}
\vspace*{\fill}
\FloatBarrier
\begin{table}[H]
 \caption{Comparison to some common invertible models.} \label{tab:compare}
\begin{center}
		\begin{adjustbox}{max width=\textwidth}
    \begin{tabular}{c | c c c c c c | c}
				\Xhline{2\arrayrulewidth}\bigstrut
        Property & NICE & Real-NVP & Glow & MaCow & FFJORD & i-ResNet & MintNet\\
				\Xhline{1\arrayrulewidth}\bigstrut
        Analytic Forward & \cmark & \cmark & \cmark & \cmark & \xmark & \cmark & \cmark\\
        Analytic Inverse & \cmark & \cmark & \xmark & \xmark & \xmark & \xmark & \xmark\\
        Non-volume Preserving & \xmark & \cmark & \cmark & \cmark & \cmark & \cmark & \cmark \\
        Exact Likelihood & \cmark & \cmark & \cmark & \cmark & \xmark & \xmark & \cmark\\
	    \Xhline{2\arrayrulewidth}
    \end{tabular} 
    \end{adjustbox}
\end{center}
\end{table}

\begin{table}[H]
 \caption{Sampling time for 64 samples for MintNet, i-ResNet and autoregressive method on the same model architectures. The time is evaluated on a NVIDIA TITAN Xp.} \label{tab:sample_time}
\begin{center}
    \begin{tabular}{p{5cm} c c c}
        \toprule
        Method & MNIST & CIFAR-10 & ImageNet 32$\times$32\\
        \midrule
        i-ResNet~\cite{i-resnet} (100 iterations) 
&11.56s &99.41s &92.53s\\ 
        Autoregressive (1 iteration) &63.61s &2889.64s &2860.21s\\
        \midrule
        MintNet (120 iterations) (ours) 
&12.81s  &117.83s  &120.78s\\
        \bottomrule
    \end{tabular} 
\end{center}
\end{table}
\FloatBarrier
\vfill
\clearpage
\section{More Samples}
\label{app:samples}
In this section, we provide more uncurated MintNet samples on MNIST, CIFAR-10 and ImageNet 32$\times$32.
\vspace*{\fill}
\FloatBarrier
\begin{figure}[H]
\centering
\includegraphics[width=0.70\textwidth]{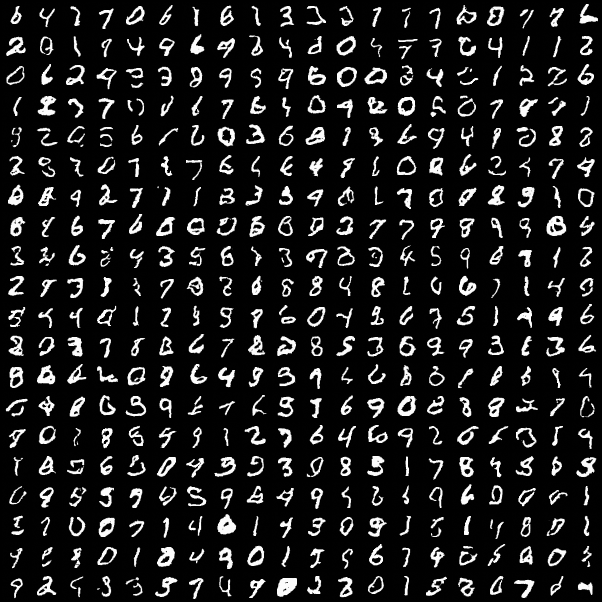}
\caption{MintNet MNIST samples.}
\label{fig:samples_mnist}
\end{figure}
\FloatBarrier
\vfill
\vspace*{\fill}
\FloatBarrier
\begin{figure}
\centering
\includegraphics[width=0.70\textwidth]{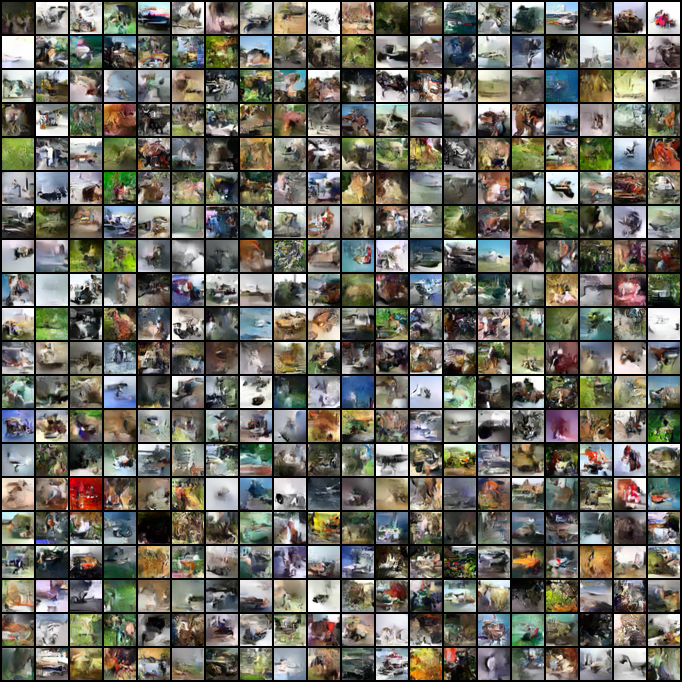}
\caption{MintNet CIFAR-10 samples.}
\label{fig:samples_cifar10}
\end{figure}

\begin{figure}
\centering
\includegraphics[width=0.7\textwidth]{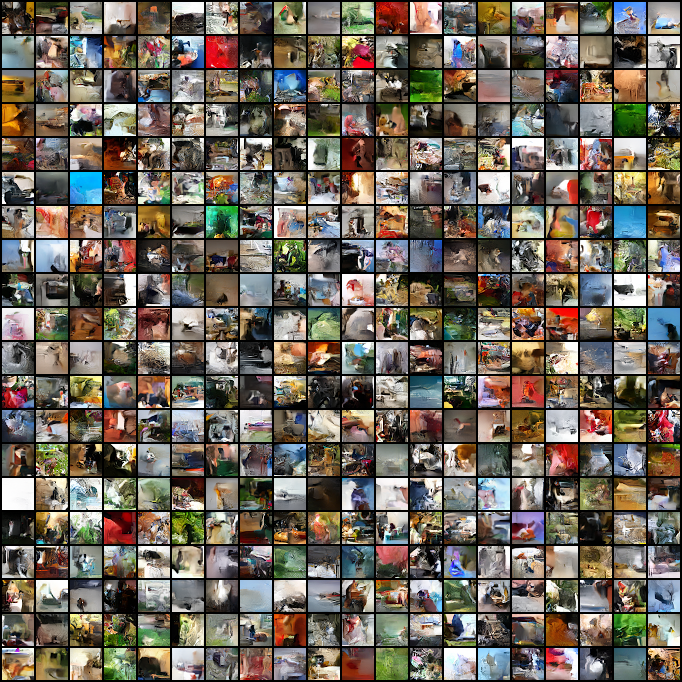}
\caption{MintNet ImageNet 32$\times$32 samples.}
\label{fig:samples_imagenet32}
\end{figure}
\FloatBarrier
\vfill